\documentclass[letterpaper, 10 pt, conference]{ieeeconf}  

\IEEEoverridecommandlockouts                              

\usepackage[T1]{fontenc}
\usepackage{preamble}

\title{\LARGE \bf
Collaborative Graph Exploration with Reduced Pose-SLAM Uncertainty via Submodular Optimization
}

\author{
Ruofei Bai$^{1, 2}$, Shenghai Yuan$^{1}$, Hongliang Guo$^{2}$, Pengyu Yin$^{1}$, Wei-Yun Yau$^{2}$ and Lihua Xie$^{1}$, \emph{Fellow}, \emph{IEEE}
\thanks{$^{1}$School of
Electrical and Electronic Engineering, Nanyang Technological University, Singapore 639798
        {\tt\small \{ruofei001, shyuan, pengyu001, elhxie\}@ntu.edu.sg}}%
\thanks{$^{2}$Institute for
Infocomm Research (I2R), Agency for Science, Technology and Research
(A*STAR), Singapore 138632
        {\tt\small \{stubair, guo\_hongliang, wyyau\}@i2r.a-star.edu.sg}}%
}

\begin{document}

\maketitle
\thispagestyle{empty}
\pagestyle{empty}

\begin{abstract}
This paper considers the collaborative graph exploration problem in GPS-denied environments, where a group of robots are required to cover a graph environment while maintaining reliable pose estimations in collaborative simultaneous localization and mapping (SLAM).
Considering both objectives presents challenges for multi-robot pathfinding, as it involves the expensive covariance inference for SLAM uncertainty evaluation, especially considering various combinations of robots' paths.
To reduce the computational complexity, we propose an efficient two-stage strategy where exploration paths are first generated for quick coverage, 
and then enhanced by adding informative and distance-efficient loop-closing actions, called loop edges, along the paths for reliable pose estimation.
We formulate the latter problem as a non-monotone submodular maximization problem by relating SLAM uncertainty with pose graph topology, which (1) facilitates more efficient evaluation of SLAM uncertainty than covariance inference, and (2) allows the application of approximation algorithms in submodular optimization to provide optimality guarantees.
We further introduce the ordering heuristics to improve objective values while preserving the optimality bound. 
Simulation experiments over randomly generated graph environments verify the efficiency of our methods in finding paths for quick coverage and enhanced pose graph reliability, and benchmark the performance of the approximation algorithms and the greedy-based algorithm in the loop edge selection problem. 
Our implementations will be open-source at \url{https://github.com/bairuofei/CGE}. 

\end{abstract}

\section{Introduction}

Multi-robot exploration holds immense promise in applications such as search and rescue~\cite{tian2020search}, surveillance and inspection~\cite{ishat-e-rabban_failure-resilient_2021}, autonomous mapping~\cite{tian_kimera-multi_2022}, etc. 
However, challenges arise in GPS-denied environments like indoor, urban valleys, and tunnels, where robots need to estimate their pose with onboard sensors during exploration.
Leveraging Multi-Robot SLAM~\cite{toward_c_slam_Lajoie_2021, huang_disco-slam_2022} for Exploration (MRSE) allows for enhanced spatial awareness, with superior efficiency and robustness than using a single robot.
Specifically, in pose graph-based multi-robot SLAM methods~\cite{grisetti_tutorial_2010}, robots incrementally construct individual pose graphs during exploration, which are connected by inter-robot loop closures.
By aggregating pose graphs from robots in a central server, collaborative pose estimation can be achieved by multi-robot pose graph optimization, mitigating single-robot odometry drift and ensuring consistent pose estimation among robots.

\begin{figure}[ht]
\centering
\includegraphics[width=\linewidth]{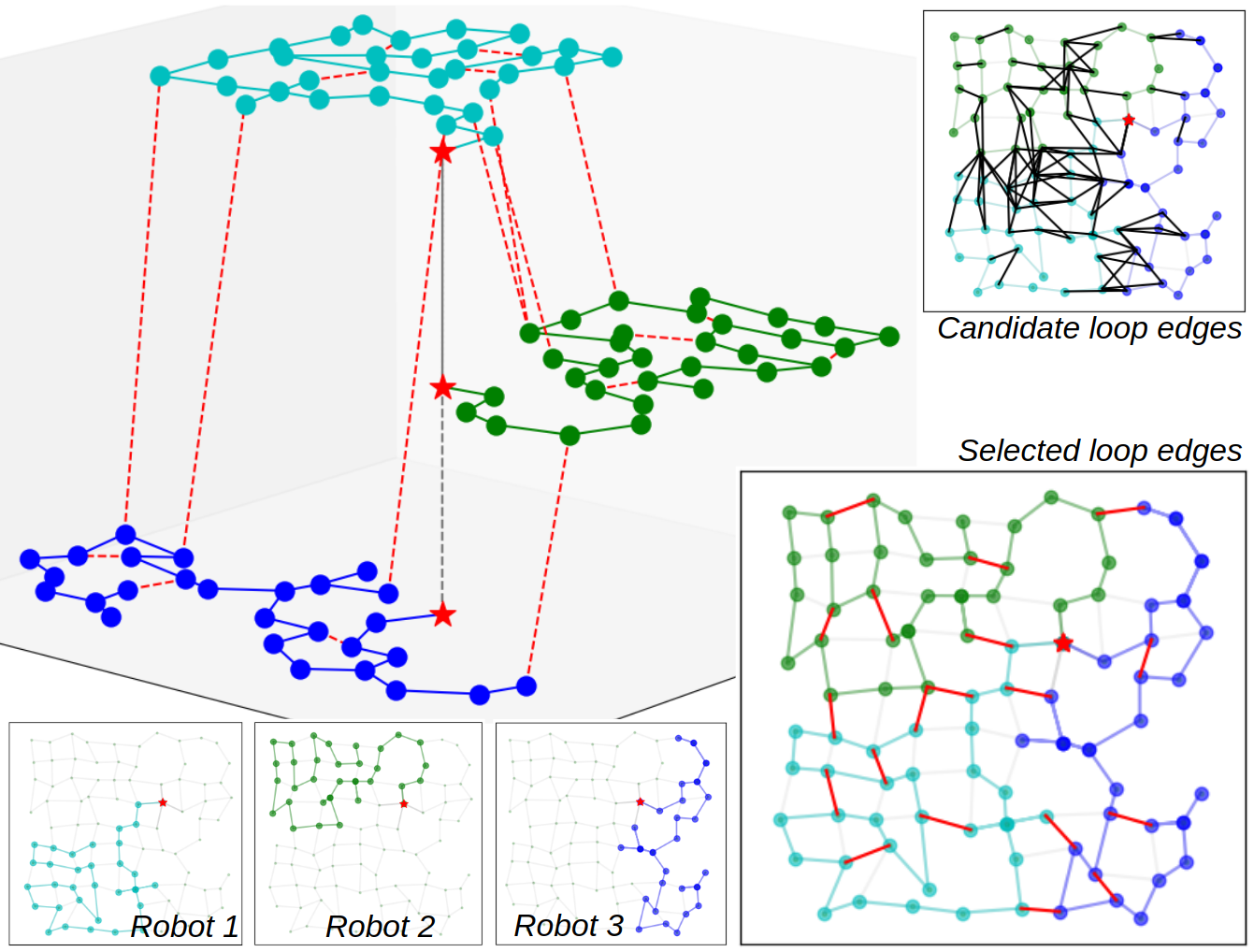}
\vspace{-15pt}
\caption{The graph exploration with three robots in a $100m\times 100m$ 2D graph environment (light gray). The robots' exploration paths are colored in cyan, green, and blue, respectively, with the starting positions marked with red stars. The selected loop edges are colored in red, and the valid candidate loop edges are colored in black.
The resulting exploration paths efficiently cover the whole graph, while forming a well-connected multi-robot pose graph topology to reduce SLAM uncertainty with informative and distance-efficient loop edges.}
 \label{fig_multilayer}
 \vspace{-10pt}
\end{figure}

However, most existing works typically treat multi-robot SLAM and exploration as distinct problems~\cite{zhang_mr-topomap_2022,kitanov_topological_2023, Yang2024Intent}, and fail to find paths that can form well-connected pose graphs, which are implicitly required by SLAM methods for reliable pose estimation~\cite{multi-tare-2023-RAL}.
Ignoring the pose graph reliability in MRSE problems may cause inconsistent pose estimations among robots due to weak connections in their individual pose graphs~\cite{DCL-sensors-2024}. 
One feasible solution is to manually define robots' trajectories to facilitate inter-robot loop closures for enhanced pose graph reliability, which, however, limits the autonomy and efficiency of robots.

The challenges of MRSE lie in the high computational complexity in the planning stage to consider both exploration efficiency and pose graph reliability; 
and the lack of 
efficient metrics to evaluate the reliability of the resulting multi-robot pose graph from various combinations of robots' candidate paths, which typically requires expensive covariance inference~\cite{vallve_active_2015} and making it intractable for efficient planning.

In this paper, we address the above challenges by formulating them into a collaborative graph exploration problem, as shown in Fig.~\ref{fig_multilayer}, where robots need to find paths to cover a graph while maintaining a well-connected multi-robot pose graph for reliable pose estimation.
To reduce the planning complexity, the pathfinding problem is divided into two stages: exploration paths are first generated for quick coverage; and then enhanced by identifying informative loop-closing actions, called \emph{loop edges}, along the obtained paths, which are inserted into the paths to improve the reliability of the resulting pose graph.
We formulate the latter problem as a distance-aware loop edge selection problem with a submodular objective function, where the pose graph reliability is approximately evaluated with a computationally more efficient graph topology metric, avoiding expensive covariance inference.
Moreover, by preserving the submodularity property of the problem, we can provide optimality guarantees to the selected loop edges using approximation algorithms in submodular optimization, with the added improvement of ordering heuristics to enhance their performance while maintaining optimality.
The contributions of this paper are summarized as follows:

\begin{itemize}
    \item  An efficient two-stage strategy to the multi-robot graph exploration problem, which prioritizes both exploration efficiency and the reliability of collaborative pose estimation in multi-robot SLAM. 
    \item A novel formulation of the informative and distance-efficient loop edge selection problem as a non-monotone submodular maximization problem, alongside the application of existing approximation algorithms with optimality guarantees.
    \item We conduct simulation experiments to confirm positive correlations between pose graph topology and SLAM uncertainty, demonstrate efficient pathfinding and reliable pose graph formation with proposed methods, and benchmark the performance and time complexity of approximation algorithms and the greedy-based algorithm.
\end{itemize}

\section{Related Works}

\subsection{Multi-Robot Active Exploration}

Multi-robot active exploration aims to find an action sequence for a group of robots to actively explore the environment, considering several factors like energy consumption, target uncertainty reduction, etc.
Existing methods in multi-robot active exploration can be generally classified into two types: search-based~\cite{atanasov_decentralized_2015} and sampling-based~\cite{kantaros_sampling-based_2021}. 
They all aim to efficiently search an action space that increases exponentially w.r.t. the number of robots and the planning horizon. 
Previous work~\cite{atanasov_decentralized_2015}
considers multi-robot active exploration to reduce the uncertainty of targets distribution in the environment, and applies the coordinate descent algorithm to provide $\frac{1}{2}$ optimality guarantee for the final selected paths.
Cai~\etal~further considers energy-aware information gathering problem for multiple robots~\cite{cai_non-monotone_2021}, where the robots' trajectories are selected by optimizing a non-monotone submodular function subject to partition matroid constraints.
However, the trajectories of robots are evaluated separately, hence cooperative sensing among robots is not considered.

The aforementioned works require covariance inference to evaluate the effect of future actions, which is computationally expensive. 
To alleviate this issue, recent works have used topology metrics to approximate the covariance inference in active SLAM~\cite{chen_broadcast_2020, placed_fast_2021}, and belief space planning~\cite{kitanov_topological_2018, kitanov_topological_2023}, which brings significantly lower computational complexity.
Specifically, Khosoussi \emph{et al.} show that the Fisher Information Matrix (FIM) in the pose graph optimization problem is related to the pose graph Laplacian matrix weighted by the corresponding covariance matrices~\cite{khosoussi_reliable_2019}. 
Chen \emph{et al.} further extend this work into 3D case~\cite{chen_cramerrao_2021}.
Additionally, 
Placed \emph{et al.} directly encapsulate the covariance matrix of each measurement as the edge weight of the Laplacian matrix, providing an efficient approximation of the FIM~\cite{placed_general_2023}.

This paper also uses the topology metrics for pose-SLAM uncertainty evaluation like in~\cite{khosoussi_reliable_2019, kitanov_topological_2018}, but further considers balancing the distance cost in a graph exploration task, which results in a non-monotone objective function.
Moreover, by utilizing a graph representation of the environment, this paper finds multiple loop edges in robots pathfinding to form a globally reliable pose graph, while related works~\cite{chen_broadcast_2020, placed_fast_2021} adopt reactive strategies that only need to find the single best loop edge within a local horizon.
This paper extends our previous work~\cite{10577228} about single-robot SLAM-aware path planning to the multi-robot case. 
Previous work provides no performance guarantees by using a simple greedy algorithm with pruning techniques for loop edge selection.
Instead, by reformulating the objective function in this paper, approximation algorithms for submodular maximization can be used to find solutions with optimality guarantees.

\section{Preliminaries}

\subsection{Submodular and Monotone Set Function}

\vspace{-2pt}
Let $f:2^{\mathcal{N}} \xrightarrow{} \mathbb{R}$ be a set function defined on a ground set $\mathcal{N}$ consisting of finite elements.
For a set $A\subseteq \mathcal{N}$ and $u\in \mathcal{N}$, we define $\Delta_{f}(u\vert A) = f(A\cup \{u\})- f(A)$ for a given set function $f(\cdot)$ as its marginal profit when adding an element $u$ to an existing set $A$. 

\begin{definition}[Submodularity]
A set function $f: 2^{\mathcal{N}} \xrightarrow{} \mathbb{R}$ is submodular if
$\Delta_{f}(u\vert A) \ge \Delta_{f}(u\vert B), \forall A\subseteq B \subseteq \mathcal{N}, u\in \mathcal{N}\backslash B.$
\end{definition}

\begin{definition}[Monotonicity]
A set function $f: 2^{\mathcal{N}} \xrightarrow{} \mathbb{R}$ is monotone if for any $A \subseteq B \subseteq \mathcal{N}$, $f(A) \le f(B)$.
\end{definition}

\subsection{Graph Laplacian Matrix}

\vspace{-2pt}
Given a connected graph $\mathcal{G}$ with $n+1$ vertices and $m$ edges, the graph Laplacian matrix is defined as:
\begin{equation}
\small
    \Lp^{\circ} = \mathbf{B}^{\circ}\mathbf{B}^{\circ\top} = \sum\nolimits_{k = 1}^{m} \mathbf{B}_k \mathbf{B}_k^{\top}\in \mathbb{R}^{(n+1)\times (n+1)},
\end{equation}
where $\mathbf{B}^{\circ}\in \mathbb{R}^{(n+1) \times m}$ is the incidence matrix, $\mathbf{B}_k$ is the $k$-th column vector of $\mathbf{B}^{\circ}$ that only has non-zero values at two indices corresponding to the vertices connected by the $k$-th edge.
The weighted Laplacian matrix of the graph $\mathcal{G}$ is defined as $    \Lp^{\circ}_{\gamma} = \mathbf{B}^{\circ}\operatorname{Diag\{\gamma_1, ..., \gamma_m\}}\mathbf{B}^{\circ\top}$, 
where $\gamma_k$ is the weight of $k$-th edge in $\mathcal{G}$.
If one vertex is anchored in $\mathcal{G}$, \ie, the corresponding row is removed from the incidence matrix $\mathbf{B}^{\circ}$, the obtained Laplacian matrix is called the reduced Laplacian matrix, denoted as $\Lp \in \mathbb{R}^{n \times n}$.
The reduced weighted Laplacian matrix $\Lp_{\gamma}$ can be defined similarly.

\subsection{Collaborative Multi-Robot Pose Graph Optimization}

In a multi-robot pose-SLAM system with a set of robots $\mathcal{R} = \{1, 2, ..., R\}$, each robot $r \in \mathcal{R}$ incrementally constructs an individual pose graph $\gpose_{r} = \langle \mathcal{X}_{r}, \mathcal{Z}_{r}  \rangle$ for pose estimation, where $\mathcal{X}_{r}$ is a set of poses that correspond to the selected keyframes from sensor readings, \eg, Lidar scans or camera frames, and $\forall x_i \in \mathcal{X}_{r}$, $x_i \in \operatorname{SE}(2)$ or $\operatorname{SE}(3)$; $\mathcal{Z}_{r} = \{\langle x_i, x_j \rangle \vert x_i, x_j \in \mathcal{X}_{r}\}$ contains relative observations between pairs of poses in $\gpose_{r}$, established by either odometry estimation or loop closure detection~\cite{grisetti_tutorial_2010}.
By sharing individual pose graphs with a central server, a collaborative multi-robot pose graph $\gpose_{\mathcal{R}} = \langle \mathcal{X}, \mathcal{Z} \rangle$ can be constructed by connecting all individual pose graphs with inter-robot loop closures, where $\mathcal{X} = \cup_{r\in\mathcal{R}} \mathcal{X}_{r}$ and $\mathcal{Z} = (\cup_{r\in\mathcal{R}} \mathcal{Z}_{r}) \cup \zinter$.
The set $\zinter$ includes all inter-robot observations among robots. 
Each $z_{ij}\in \mathcal{Z}$ is assumed to follow a Gaussian distribution $\mathcal{N}(\hat{z}_{ij}(x_i, x_j), \Sigma_{ij})$, where $\hat{z}_{ij}(x_i, x_j)$ represents the expected transformation between $x_i$ and $x_j$, $\Sigma_{ij}$ is the covariance matrix for the actual observation $z_{ij}$.
Given all relative observations $\mathcal{Z}$, the multi-robot pose graph optimization aims to find the best estimation of robots' poses $\mathcal{X}$, which is equivalent to the following nonlinear least square problem~\cite{toward_c_slam_Lajoie_2021}:
\begin{equation}
    \min_{\mathcal{X}} \mathbf{F}(\mathcal{X}) =  \sum\nolimits_{z_{ij} \in \mathcal{Z}} e_{ij}^{\top}\mathbf{\Sigma}^{-1}_{ij}e_{ij},
\label{eq_cslam_obj}
\vspace{-3pt}
\end{equation}
where $e_{ij}=z_{ij} - \hat{z}_{ij}(x_i, x_j)$.
A local optimal estimation of $\mathcal{X}$ can be found using iterative local search methods, like the Gauss-Newton method or Levenberg–Marquardt algorithm.

The Hessian matrix $\mathbf{H}$ of the pose graph optimization problem is derived as:
\begin{equation}
\small
\begin{aligned}
    \mathbf{H} 
&= \frac{1}{2} \sum\nolimits_{z_{ij} \in \mathcal{Z}} \mathbf{J}_{ij} ^{\top} \boldsymbol{\Sigma}^{-1}_{ij} \mathbf{J}_{ij}
= \frac{1}{2} \sum\nolimits_{z_{ij} \in \mathcal{Z}} \mathbf{B}_{ij}\mathbf{B}_{ij}^{\top} \otimes \widetilde{\boldsymbol{\Sigma}}_{ij}^{-1},
\end{aligned}
\label{eq_hessian}
\end{equation}
where $\mathbf{J}_{ij}$ is the Jacobi matrix of the residual term $e_{ij}$ w.r.t. a vector of all variables in $\mathcal{X}$, $\mathbf{B}_{ij}\mathbf{B}_{ij}^{\top}$ is the Laplacian factor corresponding to the edge $z_{ij}$ in $\gpose_{\mathcal{R}}$; $\widetilde{\boldsymbol{\Sigma}}_{ij}$ is derived from the transformation of $\Sigma_{ij}$.
The second equality in Eq.~(\ref{eq_hessian}) holds because $\mathbf{J}_{k}$ only has non-zero values at indices $i$ and $j$, sharing a similar structure as the column vector $\mathbf{B}_{ij}$.
The detailed derivation of Eq.~(\ref{eq_hessian}) is referred to~\cite{placed_general_2023}.


\subsection{Relating Pose Graph Uncertainty with Graph Topology}
\label{sec_FIM_laplacian}

In practice, the Hessian matrix $\mathbf{H}$ is also known as the \emph{observed} FIM~\cite{khosoussi_novel_2014}, denoted as $\mathbb{I}(\boldsymbol{x})$.
It provides a lower bound approximation of the covariance matrix in pose graph optimization according to the Cramér–Rao bound.
Scalar functions for $\mathbf{H}$ or $\mathbb{I}(\boldsymbol{x})$ can therefore be used to quantify the pose-SLAM uncertainty, among which the D-optimal is shown to be superior in capturing the global uncertainty in all dimensions of the poses~\cite{carrillo_comparison_2012}.
The D-optimal is a scalar function from Theory of Optimal Experimental Design~\cite{pazman1986foundations} and is defined as $\dopt{\Lp} = \det(\Lp)^{\frac{1}{n}}$ for an $n\times n$ matrix.

According to~\cite{placed_general_2023}, the D-optimal of the FIM $\mathbb{I}(\boldsymbol{x})$ in pose graph optimization can be well-approximated by the D-optimal of the reduced Laplacian matrix of the pose graph with proper edge weights:
\begin{equation}
\small
\begin{aligned}
&\dopt{\mathbb{I}(\boldsymbol{x})} 
=
\dopt{\sum\nolimits_{z_{ij} \in \mathcal{Z}} \mathbf{B}_{ij} \mathbf{B}_{ij}^{\top} \otimes \widetilde{\boldsymbol{\Sigma}}_{ij}^{-1}} \\
\approx &
\dopt{\sum\nolimits_{z_{ij} \in \mathcal{Z}} \mathbf{B}_{ij} \mathbf{B}_{ij}^{\top} \cdot \dopt{\widetilde{\boldsymbol{\Sigma}}_{ij}^{-1}} } 
= 
\dopt{\Lp_{\gamma}}.
\end{aligned}
\label{eq_relationship}
\end{equation}
The subscript $\gamma$ in $\Lp_{\gamma}$ indicates that each edge $z_{ij}$ in the pose graph is weighted by 
$\small D\text{-}opt(\raisebox{-0.6ex}{$\widetilde{\boldsymbol{\Sigma}}_{ij}^{-1}$})$. 
Note that Eq.~(\ref{eq_relationship}) establishes the relationship between the pose graph reliability and graph topology, which is used in this paper for pose graph uncertainty evaluation.
It brings two main advantages: 
(1) adding loop closures to a pose graph can be directly reflected in the graph Laplacian matrix without multi-robot covariance inference, which is convenient when evaluating the informativeness of candidate loop-closing actions;
(2) the dimension of the weighted graph Laplacian only depends on the number of poses in the pose graph, regardless of the pose dimension, which is computationally more efficient.

\section{Problem Statement}

This paper considers the problem that a set of robots $\mathcal{R} = \{1, ..., R\}$ are required to explore an environment represented as a graph $\G = \langle \mathcal{V}, \mathcal{E}, \omega \rangle$, as shown in Fig.~\ref{fig_multilayer}, where $\mathcal{V}\subseteq \mathbb{R}^{2} / \mathbb{R}^{3}$ is a set of places of interests distributed in a 2D (or 3D) environment; $\mathcal{E}\subseteq \mathcal{V}\times \mathcal{V}$ includes pairs of vertices that are directly connected; $\omega: \mathcal{E} \xrightarrow{} \mathbb{R}_{\ge 0}$ is the distance function for the edges. 
The graph $\G$ is a discretized representation of an environment, which can be obtained from a prior topological map~\cite{oswald_speeding-up_2016}, an existing roadmap, or derived from the Voronoi partition of the environment~\cite{kim2020voronoi}. 
The robots are initially distributed over the vertices of $\mathcal{G}$, denoted as the set $\{v_{r}^{0}\}_{r\in \mathcal{R}}$.
The problem is to find a set of paths $\{\mathcal{P}_{r}\}_{r\in \mathcal{R}}$ starting from $\{v_{r}^{0}\}_{r\in \mathcal{R}}$, following which the robots in $\mathcal{R}$ can (1) quickly cover all vertices in $\mathcal{G}$; (2) maintain reliable multi-robot pose graph topology to reduce SLAM uncertainty.

\subsection{The Two-Stage Strategy}

The above problem directly relates to the Vehicle Routing Problem (VRP), \ie, a generalization of the Traveling Salesman Problem to multiple robots, and is generally NP-Hard~\cite{book_VRP_2002}.
It requires finding paths for several robots to visit specified locations once in the environment starting and ending at a depot, while minimizing the maximum single robot's distance.
To reduce the planning complexity and facilitate existing VRP solvers, we divide the original problem into two stages, as shown in Fig.~\ref{fig_framework}.
\begin{itemize}
    \item Stage 1: VRP pathfinding. The first stage solves a standard VRP problem over the graph $\G$ to find the shortest paths for robots for quick graph coverage. Existing VRP solvers like OR-Tools~\cite{ortools} can be used to provide sub-optimal solutions to the problem.
    \item Stage 2: Loop edge selection. Given the VRP paths, Stage 2 opportunistically finds informative and distance-efficient loop-closing actions to be inserted into the paths to form a reliable multi-robot pose graph topology.
\end{itemize}

The two-stage strategy sequentially considers both exploration efficiency and pose graph reliability, avoiding intractable search that considers all possible combinations of robots' paths.
Specifically, in Stage~$1$, the VRP problem is defined over $\G$. 
The starting vertices of robots are defined as their initial located vertices $\{v_{r}^{0}\}_{r\in \mathcal{R}}$, while the ending vertices are not specified to encourage quick coverage.
The obtained VRP path for robot $r\in \mathcal{R}$ is denoted as $\pvrp_{r}$, which is a list of vertices in $\mathcal{G}$. 
The robots can follow the VRP paths $\{\pvrp_{r}\}_{r\in \mathcal{R}}$ to achieve quick coverage of $\G$.
Then we focus on finding informative loop edges to enhance the VRP paths in Stage~$2$, which will be introduced in the following subsections.

\vspace{-5pt}
\subsection{Simulating Abstracted Pose Graph}
\label{sec_simulate_pose_graph}

\begin{figure}[t]
\vspace{6pt}
\centering\includegraphics[width=\linewidth]{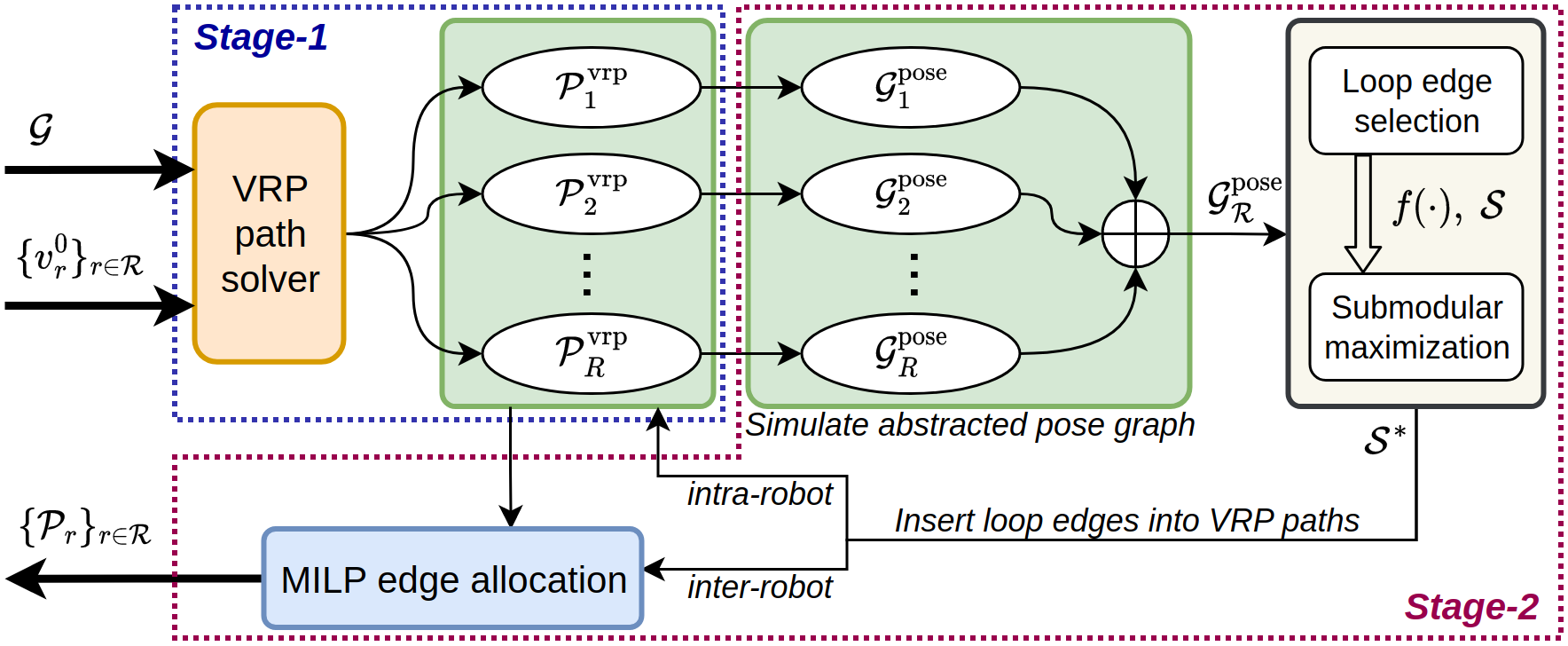}
\caption{The framework of the proposed method, which takes inputs of a graph representation of the environment $\G$ and robots' initial positions $\{v^{0}_{r}\}_{r\in \mathcal{R}}$, and finally outputs the robots' paths $\{\mathcal{P}_{r}\}_{r\in\mathcal{R}}$ that can cover the environment while resulting in a well-connected multi-robot pose graph.}
 \label{fig_framework}
 \vspace{-10pt}
\end{figure}

This section describes how to simulate an \emph{abstracted} multi-robot pose graph as the robots follow their VRP paths to explore the graph, which will be used to find potential loop closures to enhance weak connections in the pose graph.
Specifically, a robot $r$ covers a subgraph of $\G$ by following $\pvrp_{r}$, which can be treated as a hierarchical abstraction of its actual SLAM pose graph during exploration.
In this paper, we directly use such abstracted pose graph for SLAM uncertainty evaluation, motivated by the hierarchical pose graph optimization as in~\cite{grisetti_hierarchical_2010}. 
We have the following definition to construct the abstracted pose graph.

\begin{definition}[abstracted pose graph]
Given a path $\pvrp_{r}$ over a graph $\G = \langle \mathcal{V}, \mathcal{E}, \omega \rangle$, an abstracted pose graph corresponding to $\pvrp_{r}$ is defined as $\gpose_{r} = \langle \mathcal{X}_{r}, \mathcal{Z}_{r} \rangle$, where each pose $x_i\in \mathcal{X}_{r}$ corresponds uniquely to a vertex visited by $\pvrp_{r}$, and an edge $z_{ij} = \langle x_i, x_j \rangle\in \mathcal{Z}_{r}$ exists iff the corresponding vertices of the two poses $x_i$ and $x_j$ are visited consecutively in $\pvrp_{r}$.
\label{def_abstract_pose_graph}
\end{definition}

Moreover, we define $\mathcal{M}_{\mathcal{X}}: \mathcal{X} \xrightarrow{} \mathcal{V}$ as a mapping function from robot poses to vertices in $\G$, and $\mathcal{M}_{\mathcal{Z}}: \mathcal{Z} \xrightarrow{} \mathcal{E}$ as a mapping function from edges in $\gpose_{r}$ to those in $\G$.
We have $\mathcal{M}_{\mathcal{X}}(\mathcal{X}_{r})\subseteq \mathcal{V}$, $\mathcal{M}_{\mathcal{Z}}(\mathcal{Z}_{r}) \subseteq \mathcal{E}$; 
and $\cup_{r\in \mathcal{R}}\mathcal{M}_{\mathcal{X}}(\mathcal{X}_{r}) \equiv \mathcal{V}$ as the VRP paths guarantee that the graph $\mathcal{G}$ is fully explored.
With a slight notation abuse, we define the distance metric $\omega(z_{ij}) = \omega(\mathcal{M}_{\mathcal{Z}}(z_{ij}))$, for each edge $z_{ij}\in \mathcal{Z}_r$.

\begin{remark}
If the edge lengths in the original $\mathcal{G}$ vary significantly, we can add additional vertices along the long edges in $\mathcal{G}$ to make them balanced for better approximation of the abstracted pose graph.
Meanwhile, the connectivity information in $\mathcal{G}$ remains unchanged. 
\end{remark}

\begin{remark}
The orientation of poses in the abstracted pose graph is not properly defined in Def.~\ref{def_abstract_pose_graph}, because $\mathcal{V}\subseteq \mathbb{R}^{2} / \mathbb{R}^{3}$ for vertices in $\G$ and $\pvrp_{r}$, but $\mathcal{X}_{r}, \mathcal{Z}_{r}\subseteq \operatorname{SE}(2) / \operatorname{SE}(3)$.
However, we claim that the exact value of the orientation has no impact on the pose graph uncertainty evaluation, as shown in Eq.~(\ref{eq_relationship}).
\end{remark}

Additionally, we have the following assumption to capture the inter-robot loop closures formed by the VRP paths.

\begin{assumption}
    If two robots visit the same vertex, an inter-robot loop closure is assumed to be established between the two corresponding poses in their pose graphs.
\label{assump_inter_loop}
\end{assumption}

With assumption~\ref{assump_inter_loop}, an abstracted collaborative pose graph $\gpose_{\mathcal{R}} = \langle \mathcal{X}, \mathcal{Z} \rangle$ can be constructed by connecting all abstracted pose graphs $\{\gpose_{r}\}_{r\in\mathcal{R}}$ with inter-robot loop closures. 
We define the mapping function $\mathcal{M}_{\mathcal{R}}: \mathcal{X} \xrightarrow{} \mathcal{R}$ that maps a pose $x_i\in \mathcal{X}$ to its corresponding robot.
Note in the following text, we may omit the term "abstracted" when there is no ambiguity.

\vspace{-5pt}
\subsection{Formulation of Loop Edge Selection Problem}

With the abstracted collaborative pose graph $\gpose_{\mathcal{R}}$, we can then identify \emph{informative} and \emph{distance-efficient} loop-closing actions, called loop edges, over $\gpose_{\mathcal{R}}$ to reduce the pose estimation uncertainty in multi-robot SLAM.
We define candidate loop edges in $\gpose_{\mathcal{R}}$ as follows.

\begin{definition}[Loop edge]
A candidate loop edge connects two poses that are not directly connected in $\gpose_{\mathcal{R}}$. Given $\gpose_{\mathcal{R}} = \langle \mathcal{X}, \mathcal{Z} \rangle$, the set of all candidate loop edges is defined as $\mathcal{S} = \{\langle x_i, x_j \rangle \mid x_i, x_j \in \mathcal{X}; i < j; \langle x_i, x_j \rangle \notin \mathcal{Z}\}$, where $\langle x_i, x_j \rangle$ is also denoted as $z_{ij}$ for simplicity.
\label{def_loop_edge}
\end{definition}

The covariance matrix $\Sigma_{ij}$ attached to loop edge $z_{ij}$ is defined as a constant if no prior information about feature distribution is available, or can be set proportional to the number of features around $\mathcal{M}_{\mathcal{X}}(x_i)$ and $\mathcal{M}_{\mathcal{X}}(x_j)$ in the environment as in~\cite{chen_online_2019}.

\begin{remark}
    A loop edge $\langle x_i, x_j \rangle$ is not a loop closure, but a continuous action to establish loop closure between $x_i$ and $x_j$, \ie, the robot $\mathcal{M}_{\mathcal{R}}(x_j)$ will move from $\mathcal{M}_{\mathcal{X}}(x_j)$ towards another vertex $\mathcal{M}_{\mathcal{X}}(x_i)$ in $\G$ to establish loop closures.
    During this process, a chain of poses and edges may be added into $\gpose_{\mathcal{R}}$ depending on the distance between $\mathcal{M}_{\mathcal{X}}(x_i)$ and $\mathcal{M}_{\mathcal{X}}(x_j)$, rather than only one loop closure.
    Previous work~\cite{chen_broadcast_2020} has provided a bounded one-edge approximation of the chain structure based on Kirchhoff’s matrix-tree theorem~\cite{1086426} to facilitate efficient evaluation of the graph topology metric.
    However, it leads to repeated approximation when approximating multiple chain structures, \ie, multiple loop edges, and thus loses the approximation bounds.
To supplement that, here we further penalize long loop edges with their distance metrics to avoid distance-costly actions, while using one-edge approximation when evaluating their contribution to the SLAM uncertainty reduction.
\label{remark_loop_edge_approximation}
\end{remark}

After constructing a set $\mathcal{S}$ of candidate loop edges, we define the loop edge selection problem as follows:

\begin{problem}
    Given a multi-robot pose graph $\gpose_{\mathcal{R}}$ and a ground set $\mathcal{S}$ of candidate loop edges, find a subset $\mathcal{S}'\in 2^{\mathcal{S}}$ so that the following objective function is maximized:
\begin{equation}
\small
\begin{aligned}
    f(\mathcal{S}{'}) =&  \frac{1}{n} \log \det\left(\mathbf{L}(\gpose_{\mathcal{R}}) + \sum\nolimits_{z_{ij} \in \mathcal{S}{'}} \gamma_{ij} \mathbf{B}_{ij}\mathbf{B}_{ij}^{\top}\right) \\
    &- \alpha \cdot \sum\nolimits_{z_{ij}\in \mathcal{S}{'}} 2\cdot \omega(z_{ij}) + d^{\text{max}},
\end{aligned}
\label{eq_obj}
\end{equation}
where $f:2^{\mathcal{S}} \xrightarrow{} \mathbb{R}$ is a set function defined over the ground set $\mathcal{S}$; $\mathbf{L}(\gpose_{\mathcal{R}})$ is the reduced weighted Laplacian matrix corresponding to $\gpose_{\mathcal{R}}$, with starting poses of robots anchored; $\alpha$ is a parameter that balances the graph topology metric and the distance metric; $d^{\text{max}}$ is a constant that keeps $f(\cdot)$ being positive, and is defined as $d^{\text{max}} = 2\cdot \max \{ \omega(z_{ij}) \mid z_{ij}\in \mathcal{S} \} \cdot \vert \mathcal{S}\vert$.
\label{prob_edge_selection}
\end{problem}

In Problem~\ref{prob_edge_selection}, the first term $\frac{1}{n}\log \det (\mathbf{L}(\cdot))$ in $f(\cdot)$ quantifies the multi-robot SLAM uncertainty after adding a set $\mathcal{S}'$ of loop edges into $\gpose_{\mathcal{R}}$.
It is a pose graph topology metric that directly relates to the FIM in pose graph optimization, as introduced in Sec.~\ref{sec_FIM_laplacian}.
By selecting a loop edge $z_{ij}$, an edge connecting $x_i$ and $x_j$ is added to $\gpose_{\mathcal{R}}$.
Hence the Laplacian matrix is updated as $\mathbf{L}(\gpose_{\mathcal{R}}) + \gamma_{ij}\mathbf{B}_{ij}\mathbf{B}_{ij}^{\top}$, where $\gamma_{ij}$ is the encapsulated weight for the edge $z_{ij}$ with covariance $\Sigma_{ij}$.
Each selected loop edge $z_{ij}$ introduces extra distance cost of $2\cdot \omega(z_{ij})$, where the multiplier $2$ comes from an assumption that a robot will follow the path $\langle x_j, x_i, x_j \rangle$ to establish loop closures with $x_i$ and then go back to $x_j$ to continue exploration.

The objective function $f(\cdot)$ in Problem~\ref{prob_edge_selection} encourages the addition of loop edges into the robot's VRP paths, to improve the graph topology metric, thereby reduce the pose-SLAM uncertainty. Meanwhile, it also penalizes distance-costly loop edges to maintain quick coverage.
Problem~\ref{prob_edge_selection} is similar to the edge selection problem in~\cite{khosoussi_reliable_2019}, but further considers the distance cost, in which case a simple greedy algorithm as in~\cite{khosoussi_reliable_2019} has no optimality guarantee.
We will introduce the algorithms to solve Problem~\ref{prob_edge_selection} in Sec.~\ref{sec_methods}.

\subsection{Design of Parameter $\alpha$ in Problem~\ref{prob_edge_selection}}
\label{sec_alpha}

The parameter $\alpha$ in Eq.~(\ref{eq_obj}) relates to multiple factors, including the topology of the graph $\G$, the environmental area covered by $\G$, and the distribution of potential loop edges in $\gpose_{\mathcal{R}}$, which presents challenges to provide a closed-form definition for $\alpha$.
Instead, we employ a more practical numerical approach to derive $\alpha$ in this work.
Specifically, we define two bounds $\alpha^{\text{max}} =  \max_{z_{ij}\in \mathcal{S}}\frac{\Delta_{f}(z_{ij}\vert \emptyset)}{2\cdot \omega(z_{ij})}$, $\alpha^{\text{min}} =  \min_{z_{ij}\in \mathcal{S}}\frac{\Delta_{f}(z_{ij}\vert \emptyset)}{2\cdot \omega(z_{ij})}$, which corresponds to the two loop edges in $\mathcal{S}$ that have the most and the least contributions to the objective function $f(\cdot)$. 
A reasonable $\alpha$ should be within the interval $[\alpha^{\text{min}}, \alpha^{\text{max}}]$ because:
(1) if $\alpha > \alpha^{\text{max}}$, no loop edge in $\mathcal{S}$ can improve the objective value $f(\emptyset)$ because of the submodularity (proved in Sec.~\ref{sec_proof_submodularity});
(2) if $\alpha < \alpha^{\text{min}}$, all candidate loop edges in $\mathcal{S}$ will be taken into consideration, even those connecting two poses that are far from each other which, however, should be discarded.
We finally define $\alpha$ as:
\begin{equation}
\label{eq_def_alpha}
    \alpha = \alpha^{\text{min}} + \lambda (\alpha^{\text{max}} - \alpha^{\text{min}}),
\end{equation}
where $\lambda\in (0, 1)$ controls how many candidate loop edges in $\mathcal{S}$ is valid in Problem~\ref{prob_edge_selection}, \eg, a candidate loop edge $z_{ij}$ with $\frac{\Delta_{f}(z_{ij}\vert \emptyset)}{2\cdot \omega(z_{ij})} \le \alpha$ will be discarded; and influences the number of finally selected loop edges in Problem~\ref{prob_edge_selection}.

\subsection{MILP-based Loop Edge Allocation and Insertion}
\label{sec_allocation}

After solving Problem~\ref{prob_edge_selection}, the set of selected loop edges is denoted as $\mathcal{S}^{*}$, which will be allocated to robots and inserted into their VRP paths.
First, for each loop edge $z_{ij}\in \mathcal{S}^{*}$ that satisfies $\mathcal{M}_{\mathcal{R}}(x_i) = \mathcal{M}_{\mathcal{R}}(x_j)$, $\ie$, $z_{ij}$ aims to establish intra-robot loop closure of robot $\mathcal{M}_{\mathcal{R}}(x_i)$, a sequence $\langle x_j, x_i, x_j \rangle$ is directly inserted into the corresponding position in the robot's VRP path.
Second, for all $z_{ij}$ that satisfies $\mathcal{M}_{\mathcal{R}}(x_i) \ne \mathcal{M}_{\mathcal{R}}(x_j)$, 
the loop edge is allocated to either $\mathcal{M}_{\mathcal{R}}(x_i)$ or $\mathcal{M}_{\mathcal{R}}(x_j)$, depending on the results of balancing distance cost among robots.
Specifically, all such loop edges are allocated to the related robots by solving a mixed integer linear program (MILP) that minimizes the maximum distance of each involved robot, and then inserted into their allocated robots' VRP paths as in the first case. 
The details are omitted here due to space limitations.
The final paths for multi-robot graph exploration can then be obtained, denoted as $\{\mathcal{P}_{r}\}_{r\in\mathcal{R}}$.
The detailed formulation of MILP is presented in Alg.~\ref{alg_milp} in the Appendix section.


\vspace{-5pt}
\section{Methodology}
\label{sec_methods}

This section first proves the submodularity of the objective function in Problem~\ref{prob_edge_selection}, and then introduces approximation algorithms in submodular maximization to find sub-optimal solutions $\mathcal{S}^{*}$ to Problem~\ref{prob_edge_selection} with optimality guarantees.

\subsection{Submodularity of Objective Function in Problem~\ref{prob_edge_selection}}
\label{sec_proof_submodularity}

\begin{proposition}
    The set function $f(\cdot)$ in Problem~\ref{prob_edge_selection} is a non-monotone submodular function.
\label{prop_submodular}
\end{proposition}
\begin{proof}
According to the objective function~(\ref{eq_obj}), adding a loop edge will increase the graph connectivity metric but will also introduce additional distance cost.
Therefore, the monotonicity of $f(\cdot)$ is not preserved. 
To prove submodularity, it is equivalent to prove $\Delta_{f}(z_{ij}\vert A) \ge \Delta_{f}(z_{ij}\vert B)$, $\forall z_{ij}\in \mathcal{S}\backslash B$ and $A\subseteq B \subseteq \mathcal{S}$.
We have
$\Delta_{f}(z_{ij}\vert A) = \frac{1}{n} \log \det(\mathbf{L}_{A} + \gamma_{ij} \mathbf{B}_{ij}\mathbf{B}_{ij}^{\top}) - \frac{1}{n} \log \det(\mathbf{L}_{A}) - \alpha \cdot \omega(z_{ij})$, 
where $\mathbf{L}_{A} = \mathbf{L}(\gpose_{\mathcal{R}}) + \sum_{\langle x_i, x_j \rangle \in A} \gamma_{ij} \mathbf{B}_{ij}\mathbf{B}_{ij}^{\top}$ for notation simplicity.
According to the matrix determinant lemma, it holds that:
\begin{equation*}
\small
\det(\mathbf{L}_{A} + \gamma_{ij} \mathbf{B}_{ij}\mathbf{B}_{ij}^{\top})
= \det(\mathbf{L}_{A}) \det(1 + \gamma_{ij} \mathbf{B}_{ij}^{\top}\mathbf{L}_{A}^{-1}\mathbf{B}_{ij}).
\end{equation*}
Thus we have:
\begin{equation*}
\small
    \Delta_{f}(z_{ij}\vert A) = \log \det(1 + \gamma_{ij} \mathbf{B}_{ij}^{\top}\mathbf{L}_{A}^{-1}\mathbf{B}_{ij}) - \alpha \cdot \omega(z_{ij}).
\end{equation*}
Similarly, 
$\Delta_{f}(z_{ij}\vert B) = \log \det(1 + \gamma_{ij} \mathbf{B}_{ij}^{\top}\mathbf{L}_{B}^{-1}\mathbf{B}_{ij}) - \alpha \cdot \omega(z_{ij})$.
To prove  $\Delta_{f}(z_{ij}\vert A) \ge \Delta_{f}(z_{ij}\vert B)$, it is sufficient to show that $\gamma_{ij} \mathbf{B}_{ij}^{\top}\mathbf{L}_{A}^{-1}\mathbf{B}_{ij} \ge \gamma_{ij} \mathbf{B}_{ij}^{\top}\mathbf{L}_{B}^{-1}\mathbf{B}_{ij}$.

Since the reduced Laplacian matrix $\Lp_{A}$ is positive definite, its inverse $\Lp_{A}^{-1}$ is also positive definite.
According to Lemma~$9$ of~\cite{khosoussi_reliable_2019}, 
for two positive definite matrix $\Lp_{A}$ and $\Lp_{B}$, 
$\Lp_{A} \succeq \Lp_{B}$ iff 
$\Lp_{B}^{-1} \succeq \Lp_{A}^{-1}$.
Since $A \subseteq B$, 
we have $\Lp_{B} \succeq \Lp_{A}$, and thus $\Lp_{A}^{-1} \succeq \Lp_{B}^{-1}$.
It can then be proved that $\gamma_{ij} \mathbf{B}_{ij}^{\top}\mathbf{L}_{A}^{-1}\mathbf{B}_{ij} \ge \gamma_{ij} \mathbf{B}_{ij}^{\top}\mathbf{L}_{B}^{-1}\mathbf{B}_{ij}$, and hence $\Delta_{f}(z_{ij}\vert A) \ge \Delta_{f}(z_{ij}\vert B)$. 
This concludes the proof of the submodularity of the function $f(\cdot)$.
\end{proof}

\subsection{Submodular Maximization with Ordering Heuristics}
\label{sec_two_algorithms}

With Prop.~\ref{prop_submodular}, the Problem~\ref{prob_edge_selection} is recognized as an unconstrained submodular maximization (USM) problem. 
Existing approximation algorithms in submodular optimization can be used to solve Problem~\ref{prob_edge_selection} with optimality guarantees.
In this paper, we apply two algorithms, \ie, \emph{doubleGreedy}~\cite{buchbinder_tight_2012} and \emph{deterministicUSM}~\cite{buchbinder_deterministic_2018} algorithms to solve the problem, both of which provide $\frac{1}{2}$ optimality guarantee\footnote{It has been proved that no approximation algorithm can provide better than $\frac{1}{2}$ approximation with polynomial times of oracle calls~\cite{feige_maximizing_2011}.}.
The two algorithms treat the objective function $f(\cdot)$ as an oracle function and query $f(\cdot)$ for objective values.
The details of the two algorithms are referred to~\cite{buchbinder_tight_2012} and~\cite{buchbinder_deterministic_2018}, respectively.

Furthermore, we introduce the ordering heuristics to the above two algorithms, as in Alg.~\ref{alg_double_greedy} and Alg.~\ref{alg_deterministic}.
The motivation comes from two observations. 
First, both doubleGreedy and deterministicUSM algorithms have no requirements on the ordering of elements in the ground set.
Second, we find that the greedy-based algorithm (introduced in Sec.~\ref{sec_simple_greedy}) usually provides better results compared with these approximation algorithms, although it has no optimality guarantees.
Therefore, we introduce ordering heuristics to facilitate the advantage of the greedy algorithm while preserving the optimality guarantees.
Specifically, for the doubleGreedy algorithm, the next loop edge is selected as the one that has the maximum contribution given existing $X_{i-1}$ (lines 3-4 of Alg.~\ref{alg_double_greedy}).
And for the deterministicUSM algorithm, the next loop edge is selected as the best loop edge given $X^{\text{max}}$, where $(X^{\text{max}}, Y^{\text{max}})$ is the pair of set that has highest probability in previous distribution $\mathcal{D}_{i-1}$ (lines 3-5 of Alg.~\ref{alg_deterministic}).

Note the ordering heuristics require $\mathcal{O}(|\mathcal{S}|)$ oracle queries in each iteration, resulting in total oracle queries of order $\mathcal{O}(\vert \mathcal{S}\vert^2)$. 
However, the number of oracle queries can be significantly reduced by the lazy-check strategy, as shown in Tab.~\ref{tab_time}.
Specifically, we can maintain a heap to store the unvisited elements in $U$, and only update the value of the top element of the heap to find the best candidate loop edge in each iteration of algorithms.

\begin{proposition}
Alg.~\ref{alg_double_greedy} and Alg.~\ref{alg_deterministic} provide $\frac{1}{2}$-optimality guarantee for Problem~\ref{prob_edge_selection}.
\label{prop_half_guarantee}
\end{proposition}

The proof of Prop.~\ref{prop_half_guarantee} follows the proofs in~\cite{buchbinder_tight_2012} and~\cite{buchbinder_deterministic_2018}, and that adding ordering heuristics does not affect the proof.

\begin{proposition}
    Alg.~\ref{alg_double_greedy} and Alg.~\ref{alg_deterministic} require $\mathcal{O}(\vert \mathcal{S} \vert ^2)$ oracle queries, and the double greedy algorithm without ordering heuristics requires $\mathcal{O}(\vert \mathcal{S} \vert)$ oracle queries.
\end{proposition}

\subsection{Simple Greedy-based Algorithm}
\label{sec_simple_greedy}

Here we also propose a simple greedy algorithm, which, however, has no optimality guarantees.
Specifically, the greedy algorithm selects the loop edge that contributes most to the objective function from $\mathcal{S}$ in each iteration and terminates until no candidate loop edge can further improve the objective value.
The lazy-check strategy can also be used to improve its time efficiency, as compared in Tab.~\ref{tab_time}.

\begin{algorithm}[t]
\small
\SetKwInOut{Input}{Input}\SetKwInOut{Output}{Output}
\SetKwInOut{Return}{Return}
\caption{doubleGreedyWithOrder($f(\cdot)$, $\mathcal{S}$)}
\label{alg_double_greedy}
Let $X_{0}\xleftarrow{} \emptyset, Y_{0} \xleftarrow{} \mathcal{S}$, $U \xleftarrow{} \mathcal{S}$.\\
\For{$i=1$ to $|\mathcal{S}|$ }{
    $u_i \xleftarrow{} \arg \max_{u\in U} \Delta_{f}(u\mid X_{i-1})$.\\
    $U \xleftarrow{} U \backslash \{u_i\}$.\\
    Let $a_i \xleftarrow{} \max\{f(X_{i-1} \cup \{u_i\}) - f(X_{i-1}), 0\}$.\\
    Let $b_i \xleftarrow{} \max\{f(Y_{i-1} \backslash \{u_i\}) - f(Y_{i-1}), 0\}$.\\
    $p\xleftarrow{}$ random variable from $[0, 1]$.\\
    \eIf{$p > a_{i}/(a_{i} + b_{i})$}{
        $X_i \xleftarrow{} X_{i-1} \cup \{u_{i}\}$, $Y_i \xleftarrow{} Y_{i-1}$.\\
    }{
        $X_i \xleftarrow{} X_{i-1}$, $Y_i \xleftarrow{} Y_{i-1} \backslash \{u_{i}\}$.\\
    }
}
\KwRet{$\mathcal{S}^{*} \xleftarrow{} X_n$.}\\
*If $a_i = b_i = 0$, we assume $a_i / (a_i + b_i) = 1$.
\end{algorithm}

\begin{algorithm}[t]
\small
\SetKwInOut{Input}{Input}\SetKwInOut{Output}{Output}
\SetKwInOut{Return}{Return}
\caption{deterministicUSMWithOrder($f(\cdot)$, $\mathcal{S}$)}
\label{alg_deterministic}
Initialize a distribution $\mathcal{D}_{0} = \{\left(1, (\emptyset, \mathcal{S})\right)\}$, $U \xleftarrow{} \mathcal{S}$.\\
\For{$i = 1$ to $|\mathcal{S}|$}{
    $(X^{\text{max}}, Y^{\text{max}}) \xleftarrow{} \arg \max_{(X, Y)\in supp(\mathcal{D}_{i-1})} P[(X, Y)]$.\\
    $u_i \xleftarrow{} \arg \max_{u\in U} \Delta_{f}(u \mid X^{\text{max}})$.\\
    $U \xleftarrow{} U \backslash \{u_i\}$.\\
    $\forall (X, Y) \in supp(\mathcal{D}_{i-1})$, let $a_{i}(X) = f(X \cup \{u_i\}) - f(X)$, $b_{i}(Y) = f(Y \backslash \{u_i\}) - f(Y)$.\\
    Find an \textbf{extreme point} solution of the following linear program problem:
    \begin{equation*}
    \tiny
    \begin{aligned}
        \mathbb{E}_{\mathcal{D}_{i-1}}[z(X, Y)a_{i}(X) + \omega(X, Y)b_{i}(Y)] &\ge 2\cdot \mathbb{E}_{\mathcal{D}_{i} - 1}[z(X, Y)b_{i}(Y)]\\
        \mathbb{E}_{\mathcal{D}_{i-1}}[z(X, Y)a_{i}(X) + \omega(X, Y)b_{i}(Y)] &\ge 2\cdot \mathbb{E}_{\mathcal{D}_{i} - 1}[\omega(X, Y)a_{i}(Y)]\\
        z(X, Y) + \omega(X, Y) = 1, z(X, Y), \omega(X, Y) &\ge 0, \forall (X, Y) \in supp(\mathcal{D}_{i-1}) 
    \end{aligned}
    \end{equation*}\\
    $\forall (X, Y) \in supp(\mathcal{D}_{i-1})$, add following to a new distribution $\mathcal{D}_{i}$:
    \begin{equation*}
    \tiny
        \begin{aligned}
            &\left\{\left (z(X, Y) \cdot P[(X, Y)], (X + u_i, Y) \right) \mid z(X, Y) > 0\right\}\\
            \cup& \left\{\left (\omega(X, Y) \cdot P[(X, Y)], (X, Y-u_i) \right) \vert \omega(X, Y) > 0\right\}
        \end{aligned}
    \end{equation*}
    
}
\KwRet{$\mathcal{S}^{*} \xleftarrow{} \arg \max_{(X, Y)\in supp(\mathcal{D}_{n})}\{f(X)\}$.}\\
*$supp(\mathcal{D})$ represents the support of the distribution $\mathcal{D}$.
\end{algorithm}

\section{Experiments}
This section simulates various 2D graph environments of different sizes for multi-robot exploration.
We use OR-Tools~\cite{ortools} as the VRP solver and the MILP solver in Sec.~\ref{sec_allocation}.
The time limit for VRP pathfinding is set as $20$ seconds.
The linear program (LP) problem in Alg.~\ref{alg_deterministic} is solved by \emph{pulp}\footnote{https://github.com/coin-or/pulp}, with an objective function defined as minimizing $0.5\cdot \sum z(X, Y) + 0.6 \cdot \sum \omega(X, Y)$.
By default, we take $\lambda = 0.3$ in Eq.~(\ref{eq_def_alpha}) to get $\alpha$ in Problem~\ref{prob_edge_selection}.
We compare the performance of five algorithms, \ie, doubleGreedy (dGre), doubleGreedy+Ordering (dGre+order), deterministicUSM (dUSM), deterministicUSM+Ordering (dUSM+order), and simpleGreedy (sGre). 
All algorithms are implemented in Python $3$ and tested on a desktop with an i9-13900 CPU and 32 GB of RAM.

\begin{table*}[!t]
\vspace{6pt}
\captionsetup{skip=-3pt} 
\caption{Running Time Comparison}
\centering
\small
\begin{tabular}{@{}l|ll|lll|lll@{}}
\toprule
Env ($m^2$)           & sGre      & sGre(l-c) & dGre~\cite{buchbinder_tight_2012}      & dGre+order & dGre+order(l-c) & dUSM~\cite{buchbinder_deterministic_2018}     & dUSM+order & dUSM+order(l-c) \\ \midrule
$60\times 60$   & $0.050$ & $\mathbf{0.017}$ & $0.061$ & $0.170$     & $0.054$          & $1.389$  & $1.373$       & $1.243$            \\
$80\times 80$   & $0.218$ & $\mathbf{0.050}$ & $0.133$ & $0.692$     & $0.183$          & $6.164$  & $6.762$       & $6.023$            \\
$100\times 100$ & $0.922$ & $\mathbf{0.139}$ & $0.296$ & $2.428$     & $0.489$          & $18.467$ & $19.816$      & $16.923$           \\
$120\times 120$ & $4.389$ & $\mathbf{0.483}$ & $0.948$ & $10.313$    & $1.733$          & $77.212$ & $72.687$      & $63.868$           \\ \bottomrule
\end{tabular}

\label{tab_time}
\raggedright
~~~~Note: The unit for all data in the table is seconds; each item is the averaged result over fifty experiments; l-c stands for lazy-check.
\end{table*}

\subsection{Random Generation of Graph Experiments}
The graph environments are randomly derived from grid-like structures of different sizes, $\ie$, $60m\times 60m$, $80m\times 80m$, $100m\times 100m$, $120m\times 120m$, with a grid step as ten meters.
An example graph covering a $100m\times 100m$ area is shown in Fig.~\ref{fig_multilayer}.
Each vertex in a grid graph is only connected to its adjacent vertices.
Then $10\%$ vertices (and their outgoing edges) and additional $3\%$ edges are randomly removed to create diverse topologies.
Finally, Gaussian noise of $\mathcal{N}(0, 2m)$ is added to the $xy$ coordinates of vertices in the graph.
The covariance matrix $\Sigma_{ij}$ for each edge in $\gpose_{\mathcal{R}}$ and the set $\mathcal{S}$ is set as $\operatorname{diag}\{0.1m, 0.1m, 0.001rad\}$.
Fifty random graphs are independently generated for each environment size.
By default, three robots starting from the same vertex are deployed in the graph exploration tasks.

\begin{figure}[!t]
\centering
\subfloat[]{\includegraphics[width=.46\linewidth]{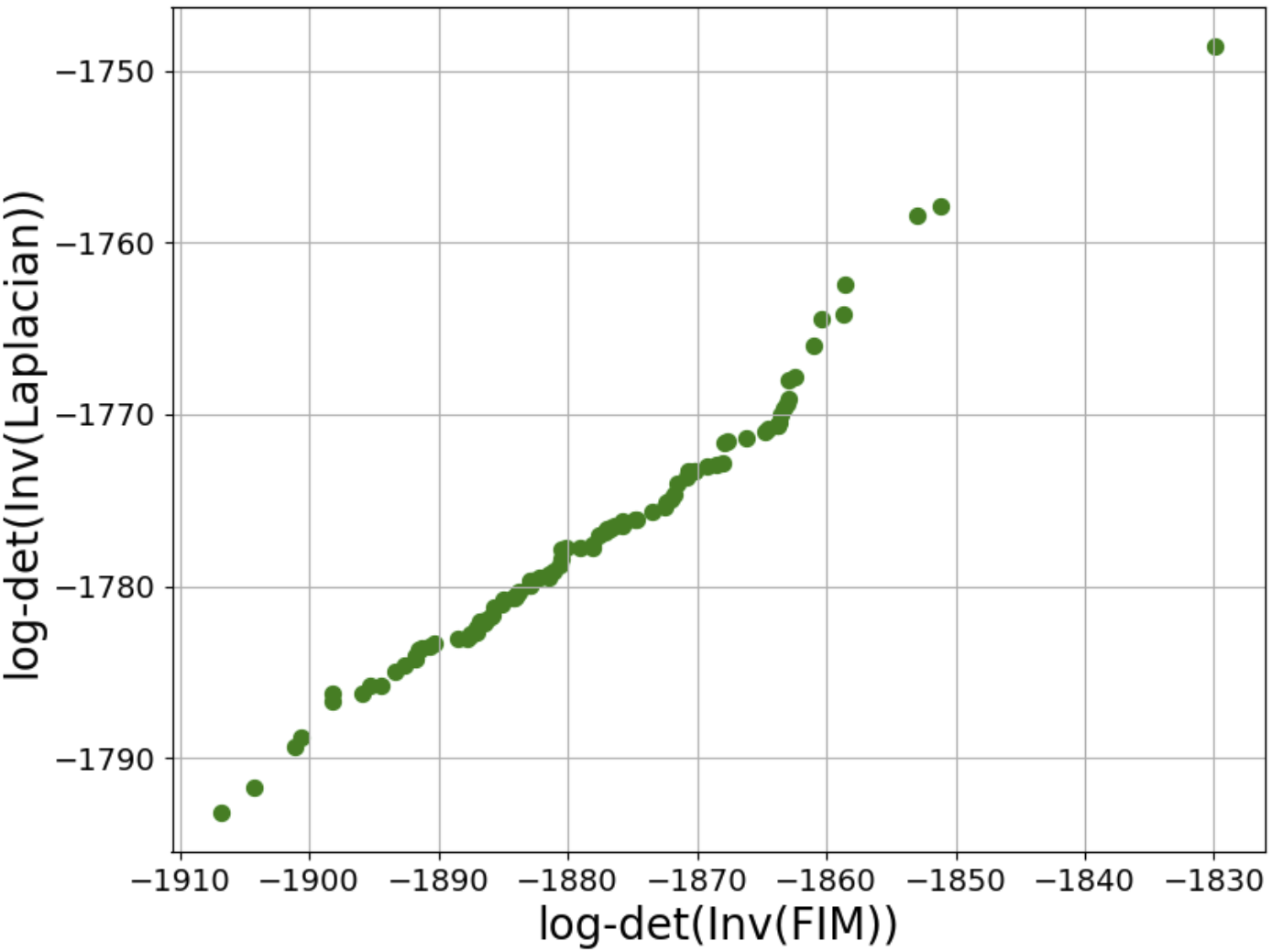}}\hspace{2pt}
\subfloat[]{\includegraphics[width=.46\linewidth]{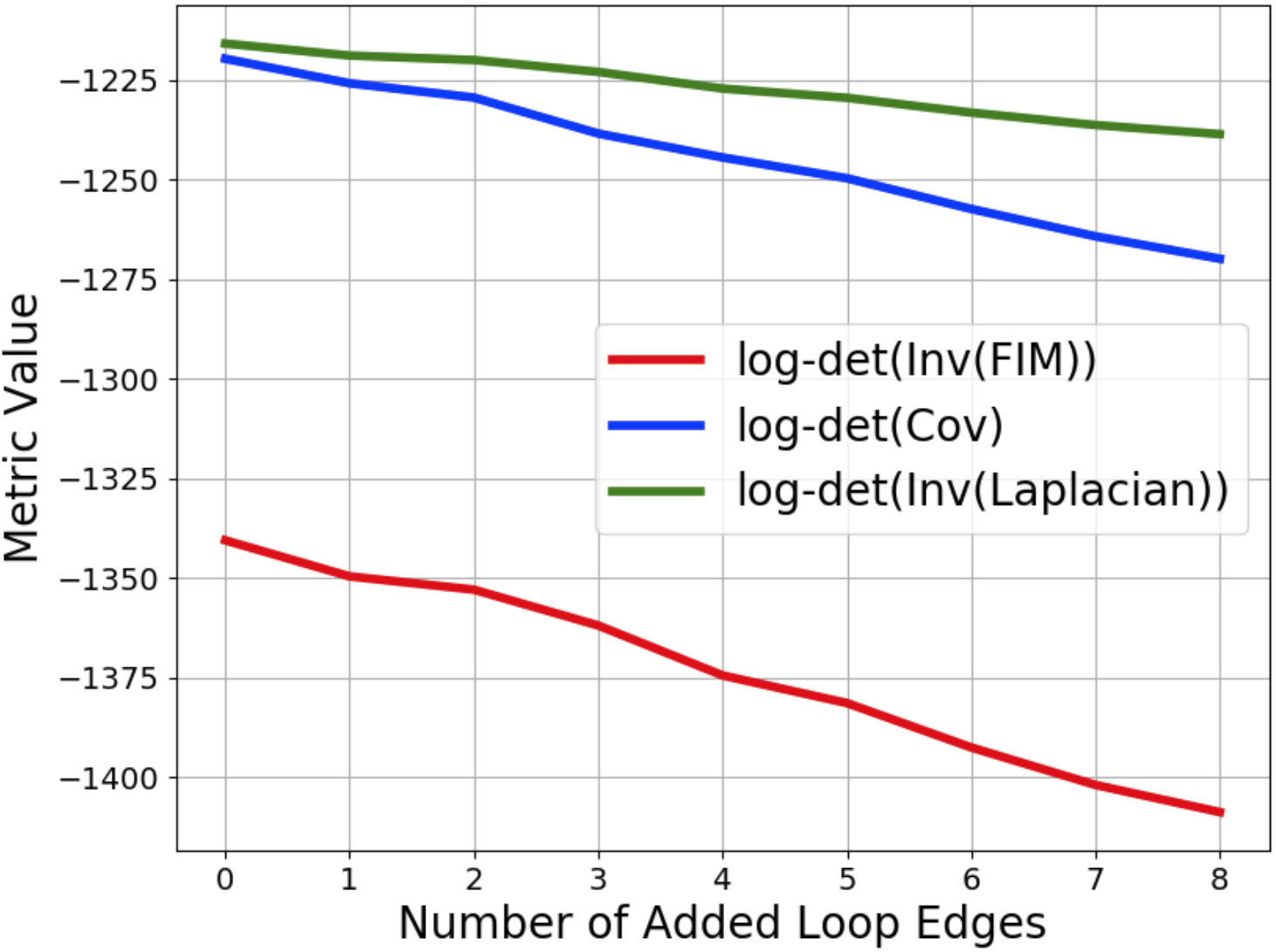}}
\vspace{-5pt}
\caption{(a) The relationship between $\log\det(\mathbf{L}_{\gamma}^{-1})$ and $\log\det(\mathbb{I}^{-1})$ evaluated on a set of pose graphs derived from a $120m\times 120m$ graph environment; (b) The pose estimation uncertainty decreases as more loop edges are added into the collaborative pose graph.}
\label{fig_approximation}
\vspace{-5pt}
\end{figure}

\vspace{-2pt}
\subsection{Relationship Verification of FIM and Graph Laplacian}

We first verify the relation between the $\log$ determinant of full FIM in multi-robot pose graph optimization and the corresponding weighted pose graph Laplacian matrix, as shown in Fig.~\ref{fig_approximation}(a).
Each point in Fig.~\ref{fig_approximation}(a) corresponds to a multi-robot pose graph with random sets of loop edges added to it.
Similar to the results in~\cite{placed_general_2023},
the two metrics have a positive correlation with each other.
Moreover, the graph topology metric $\log \det (\mathbf{L})$ preserves action consistency with the original FIM metric when evaluating candidate loop edges.
Therefore, it is reasonable to use the graph topology metric in Problem~\ref{prob_edge_selection} for pose graph uncertainty evaluation, which brings lower computational complexity than the original FIM metric.
Fig.~\ref{fig_approximation}(b) shows the changing trend of the full covariance matrix (obtained from GTSAM~\cite{gtsam}), full FIM, and the reduced weighted Laplacian matrix as more loop edges are added into $\gpose_{\mathcal{R}}$.
Without accounting for the distance cost, the three metrics decrease monotonically, indicating that the pose uncertainty in SLAM reduces as the connectivity of the collaborative pose graph improves, and thereby the graph topology metric in Problem~\ref{prob_edge_selection} encourages adding loop edges into the robots' paths.

\begin{figure}[t]
\centering
\subfloat[$60m\times 60m$]{\includegraphics[width=.475\linewidth]{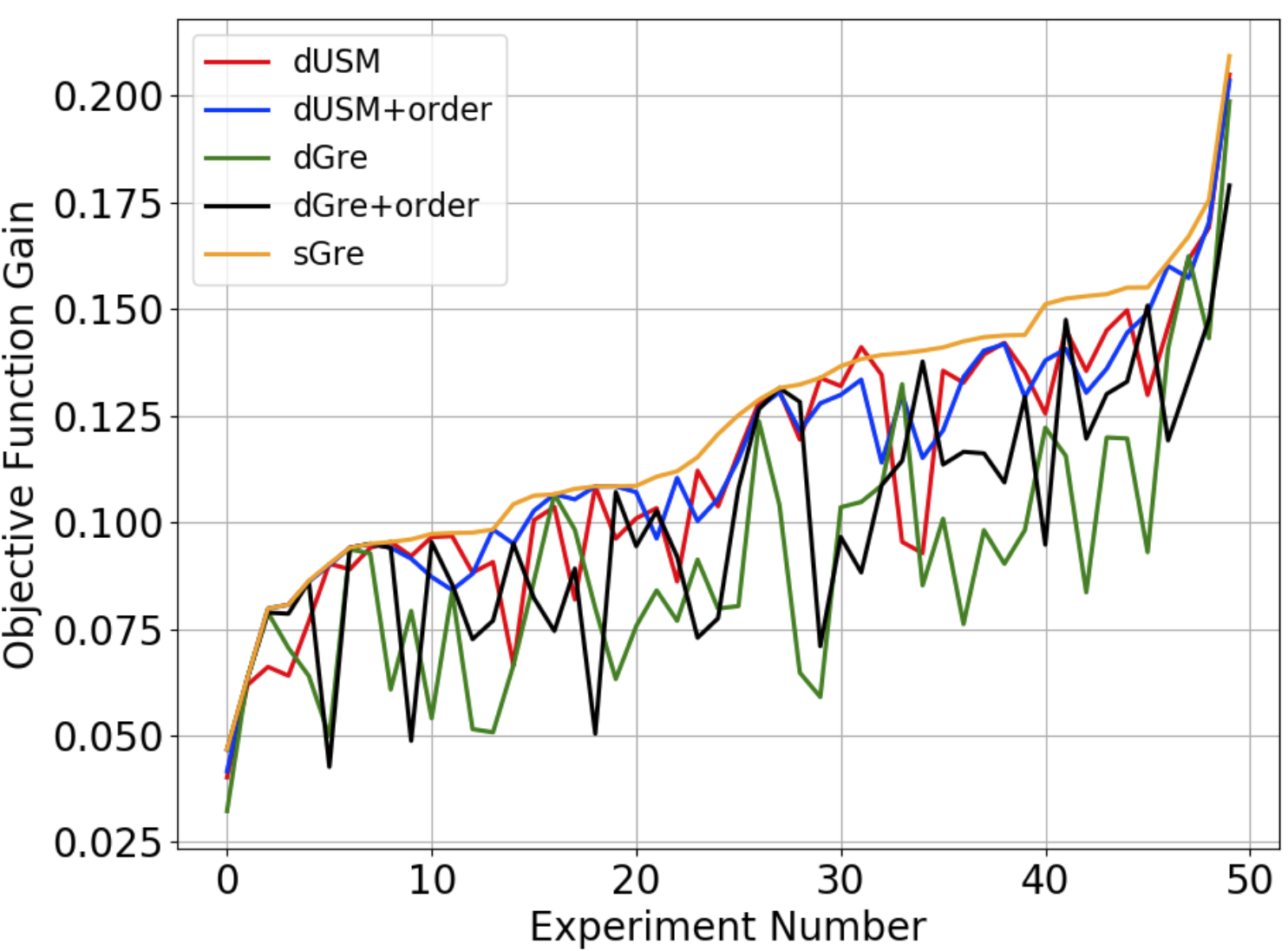}}\hspace{2pt}
\subfloat[$80m\times 80m$]{\includegraphics[width=.475\linewidth]{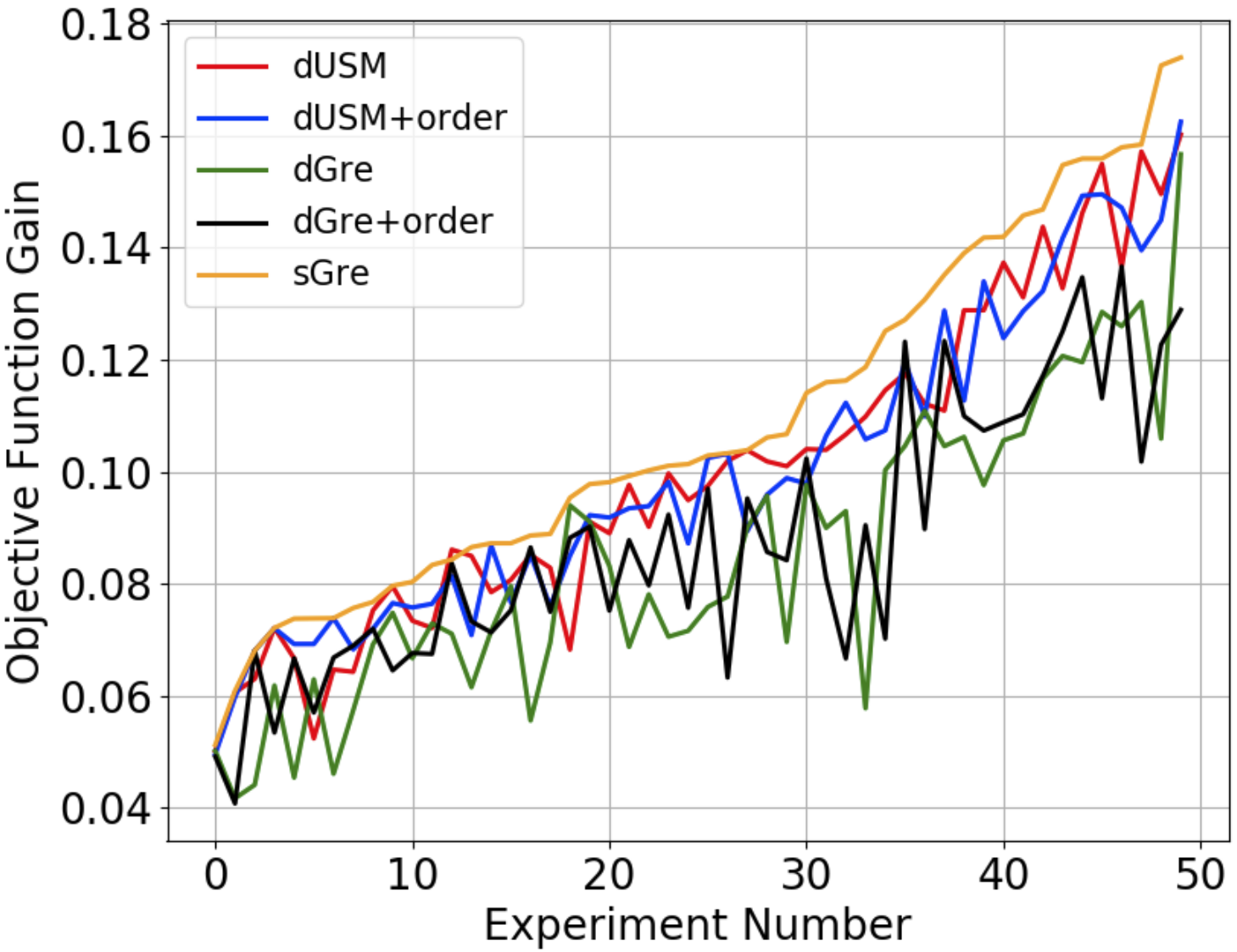}}\\
\subfloat[$100m\times 100m$]{\includegraphics[width=.475\linewidth]{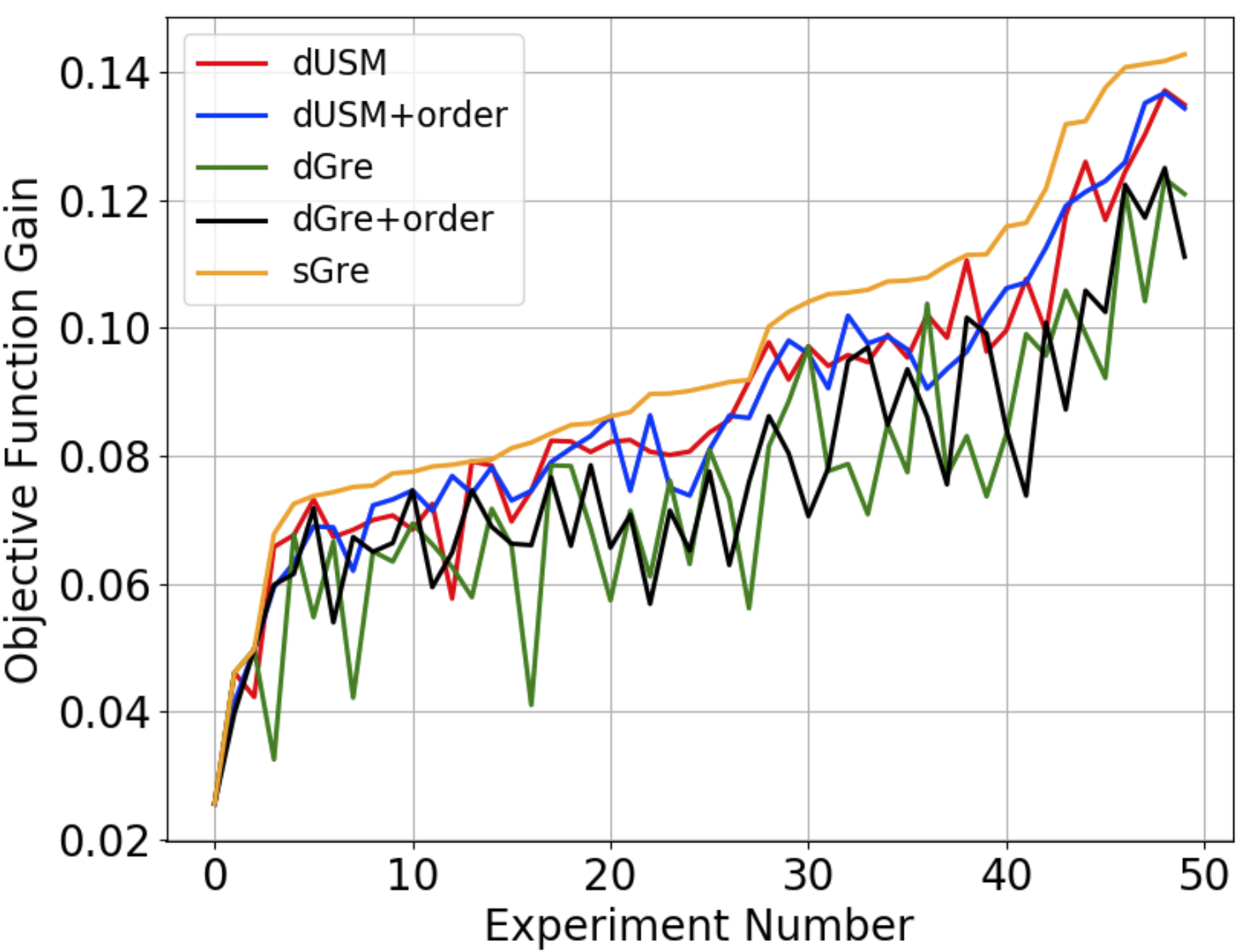}}\hspace{2pt}
\subfloat[$120m\times 120m$]{\includegraphics[width=.475\linewidth]{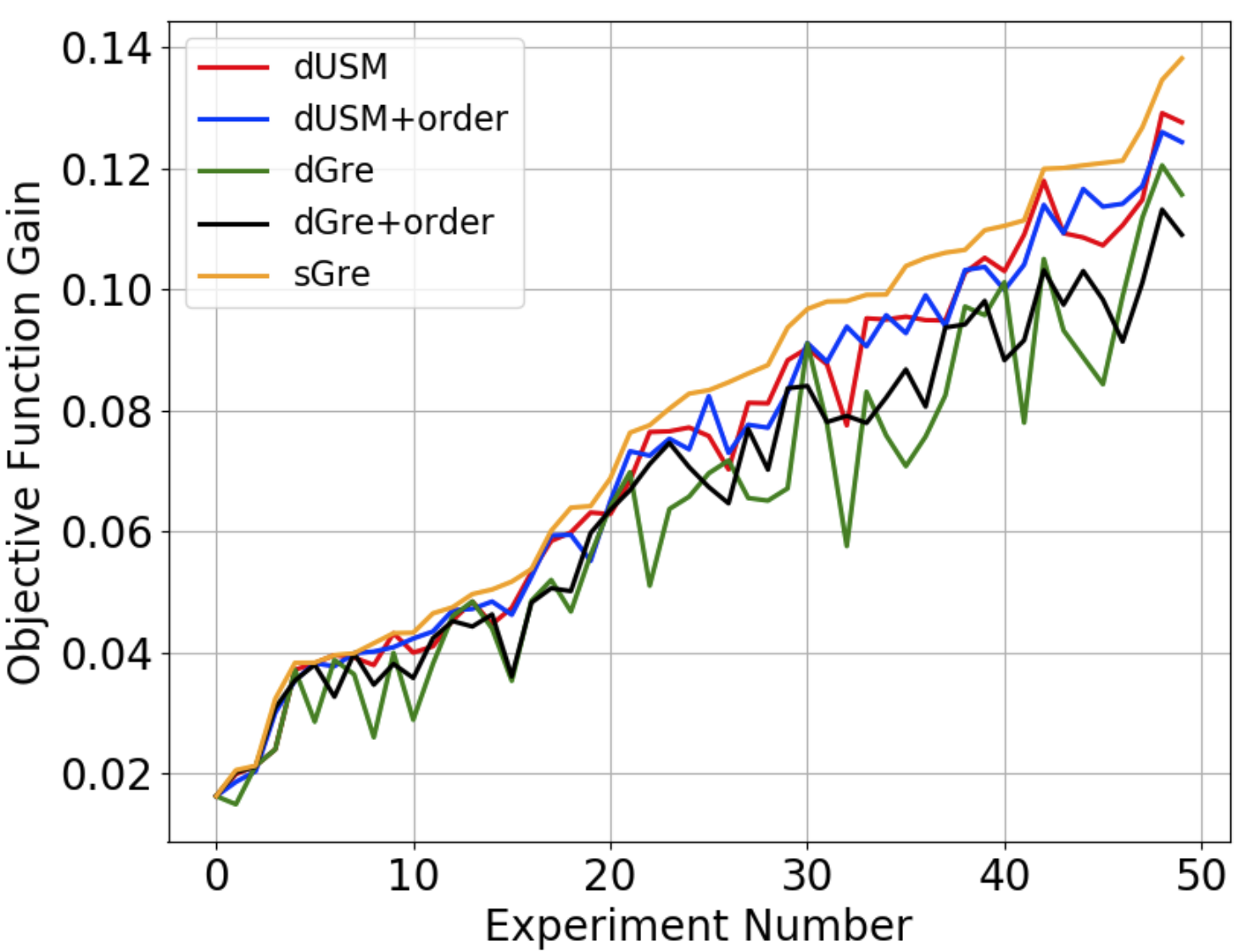}}
\caption{The objective gain of Problem~\ref{prob_edge_selection} with the five algorithms in $50$ independent experiments. 
Note the results are sorted according to the objective value of the sGre algorithm for better visualization.}
\label{fig_2performance}
\end{figure}

\begin{figure}[t]
\centering
\subfloat[]{\includegraphics[width=.47\linewidth]{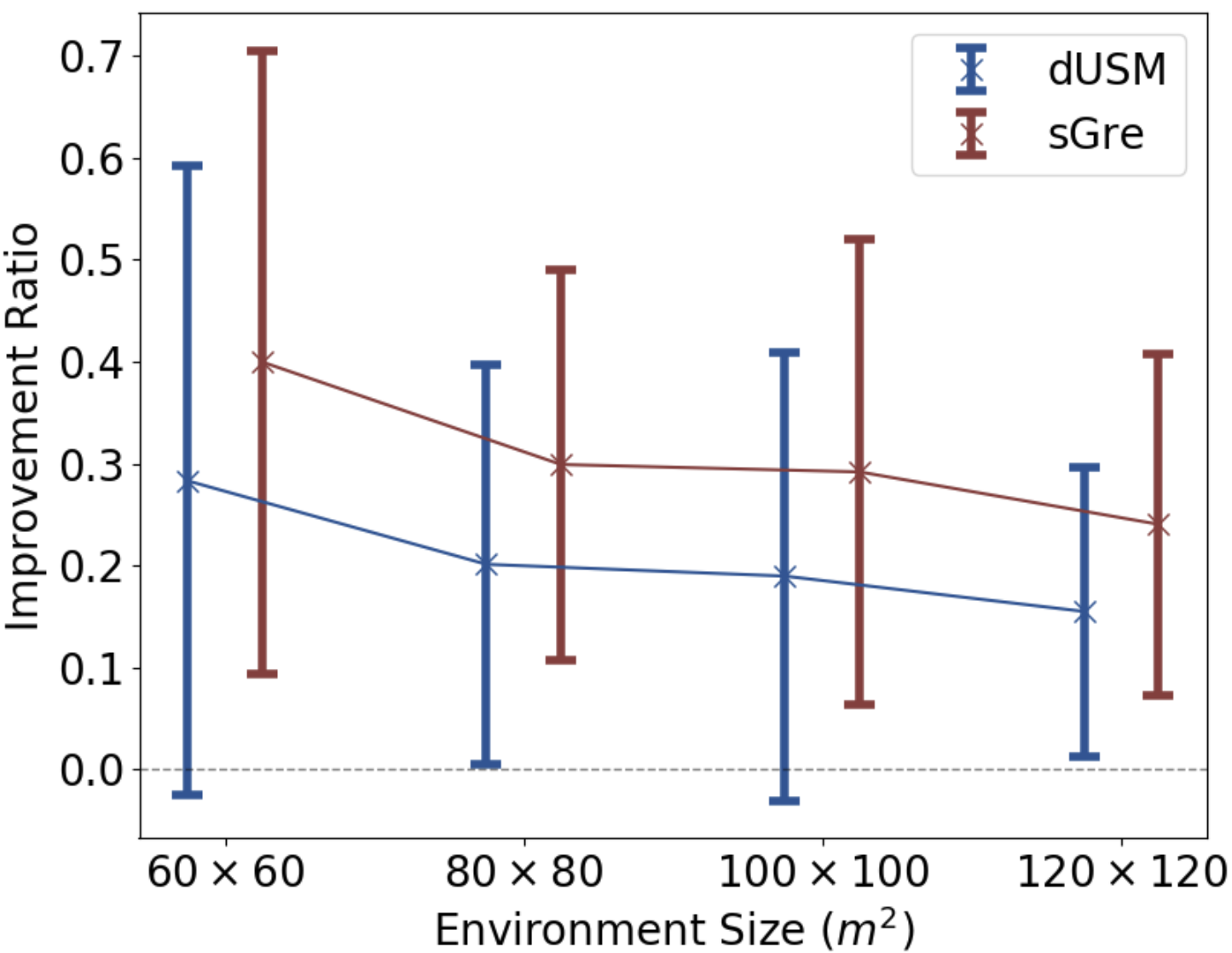}}\hspace{2pt}
\subfloat[]{\includegraphics[width=.485\linewidth]{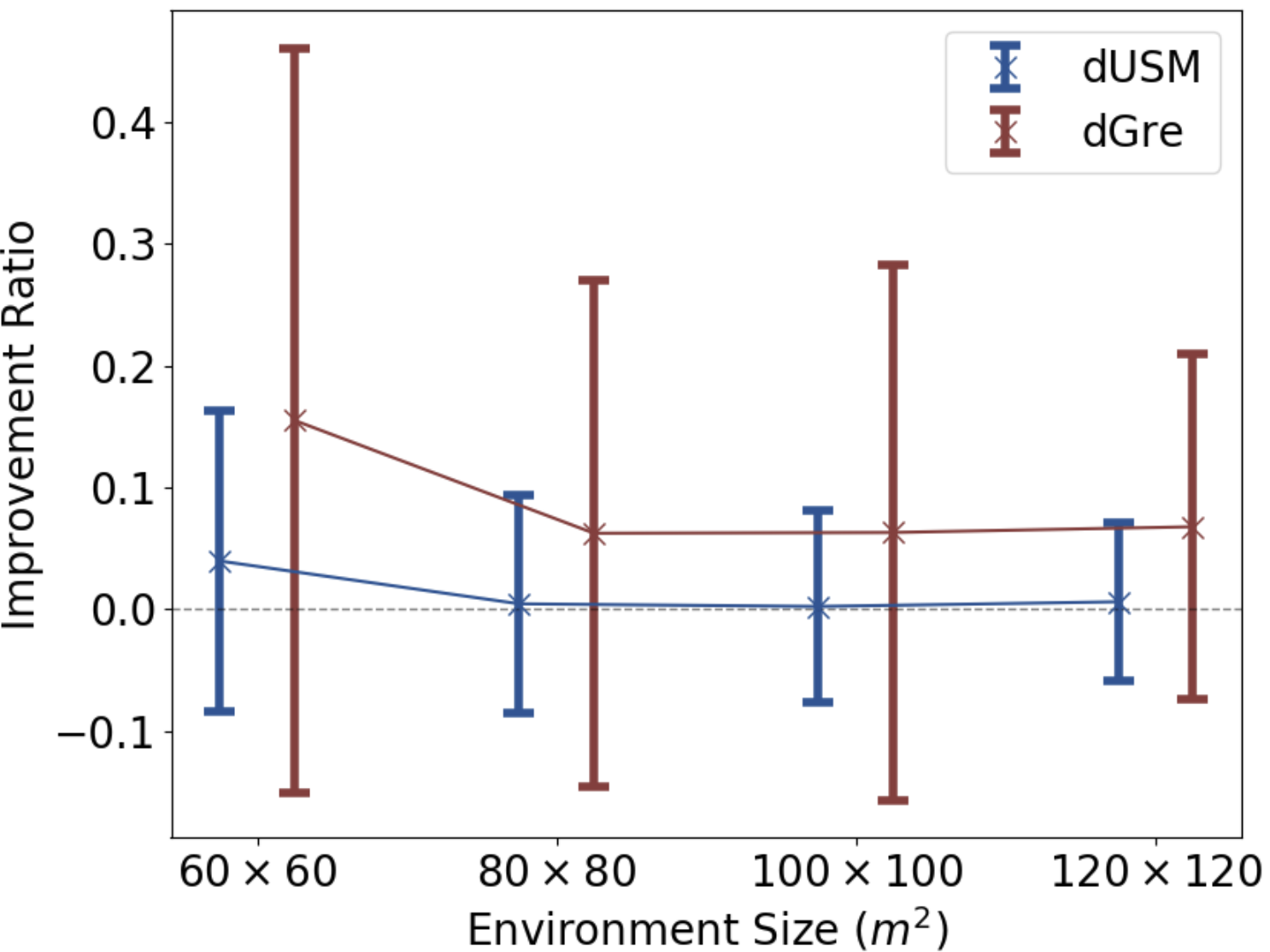}}
\caption{(a) The objective improvement ratio of the dUSM and sGre algorithms w.r.t. the dGre algorithm; (b) The objective improvement ratio after adding the ordering heuristics to dGre and dUSM algorithms.}
\label{fig_improve_ratio}
\end{figure}

\begin{figure}[!t]
\centering
\subfloat[$\lambda = 0.1$]{\includegraphics[width=.31\linewidth]{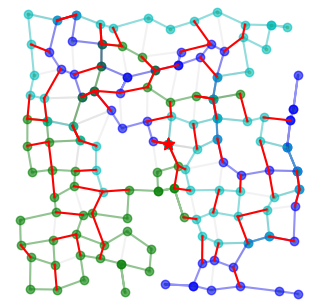}}\hspace{1pt}
\subfloat[$\lambda = 0.3$]{\includegraphics[width=.31\linewidth]{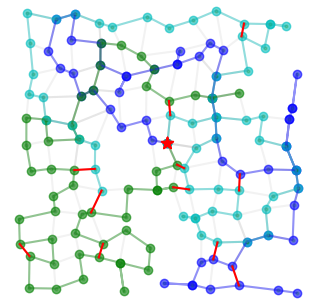}}\hspace{1pt}
\subfloat[]{\includegraphics[width=.31\linewidth]{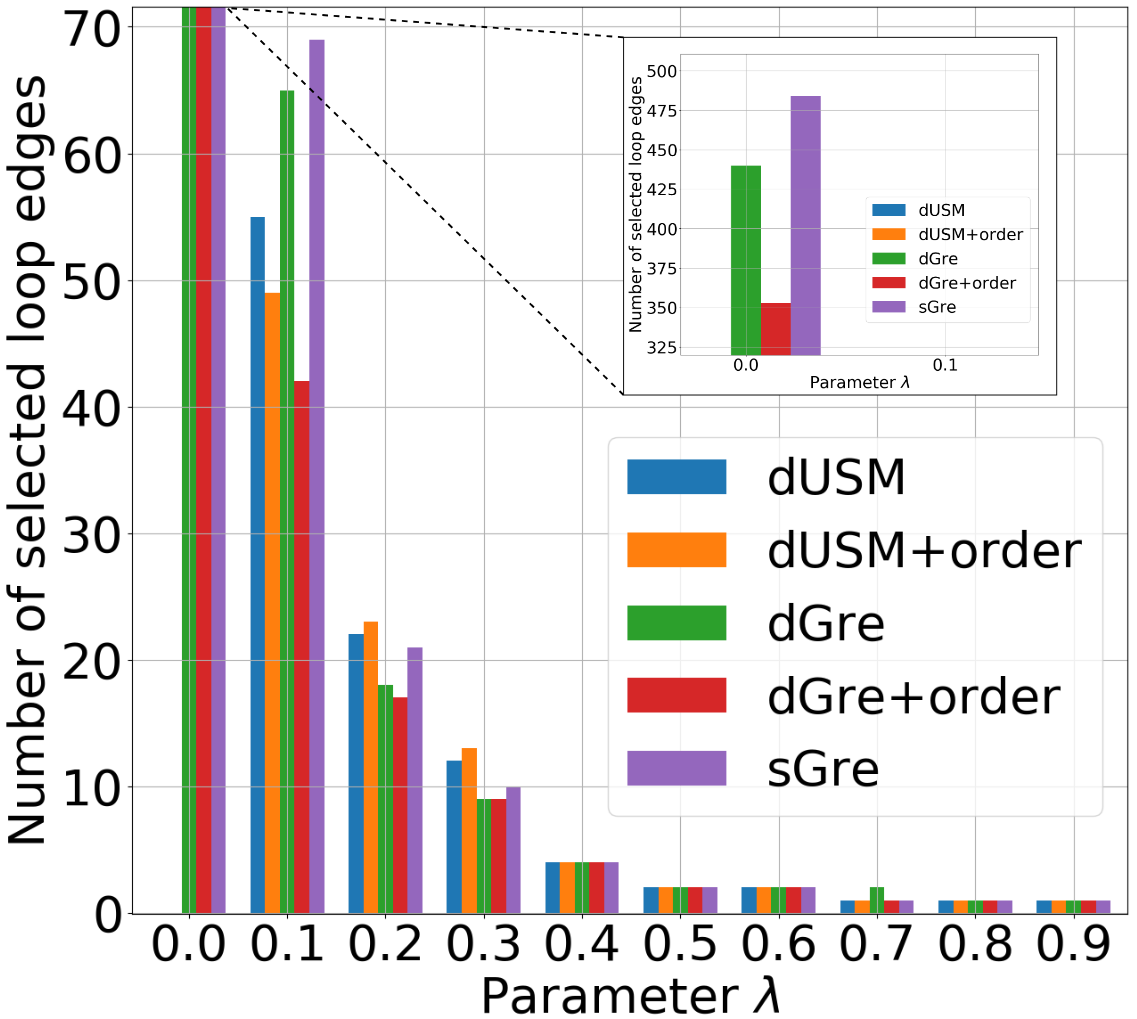}}
\caption{The loop edge selection problem defined over a $120m\times 120m$ graph environment with $11303$ candidate loop edges. Robots' paths are painted in different colors. (a) and (b) show the final selected loop edges (red) by the sGre algorithm with different $\lambda$ in Eq.~(\ref{eq_def_alpha}); (c) shows the comparison of the number of selected loop edges with different $\lambda$. Note the results of dUSM and dUSM+order are missing when $\lambda = 0$ because the two algorithms cannot output the results within an hour.}
\label{fig_lambda}
\end{figure}

\vspace{-2pt}
\subsection{Performance Comparison in Loop Edge Selection}

Fig.~\ref{fig_2performance} shows the performance of the five algorithms in loop edge selection problems over various environments.
The averaged running time of the algorithms is shown in Tab.~\ref{tab_time}.
Generally, the sGre algorithm provides the best results compared with others, although it has no performance guarantee.
For the four approximation algorithms that provide $\frac{1}{2}$ optimality guarantees, the dUSM algorithm generally gets better results than the dGre algorithm.
However, it spends significantly more time than the dGre algorithm as in Tab.~\ref{tab_time},
because it solves a linear program in each iteration, which dominates the computational time.
The dGre algorithm usually gets the worst results, as shown in Fig.~\ref{fig_improve_ratio}(a), that the objective value of sGre and dUSM algorithms are $31\%$ and $21\%$ better than the dGre algorithm respectively.
However, it runs faster than the dUSM algorithm by an order of magnitude while still providing optimality guarantees.
On average, adding ordering heuristics to the dGre algorithm can improve its performance by $9\%$, while no discernible improvement is observed for the dUSM algorithm, as shown in Fig.~\ref{fig_improve_ratio}(b).
The ordering heuristics also introduces extra time complexity, as in Tab.~\ref{tab_time}. 
However, the extra time expense can be significantly reduced by applying the lazy-check strategy introduced in Sec.~\ref{sec_two_algorithms}, which also makes the sGre algorithm the fastest among all the compared algorithms.


Additionally, we also evaluate the effect of the parameter $\lambda$ in Eq.~(\ref{eq_def_alpha}) on the number of finally selected loop edges in Problem~\ref{prob_edge_selection}, as shown in Fig.~\ref{fig_lambda}.
As $\lambda$ increases from $0$ to $1$, the number of selected loop edges decreases significantly, which controls the frequency of active loop-closing actions in the final paths, and also balances the distance cost during exploration, as we analyzed in Sec.~\ref{sec_alpha}.

\section{Conclusion and Future Work}
This paper investigates a multi-robot graph exploration problem considering both exploration efficiency and pose graph reliability in multi-robot SLAM.
A two-stage strategy is proposed that first generates exploration paths for quick graph coverage, and then improves the paths by inserting informative and distance-efficient loop edges along the path.
The latter problem is formulated as a non-monotone submodular maximization problem, hence several approximation algorithms are applied with optimality guarantees. 
We also benchmark the performance of approximation algorithms and the greedy algorithm in the loop edge selection problem.

The proposed method can be applied as a high-level path planner in real exploration tasks based on a graph representation of the environment.
Future work is to incorporate the proposed path planner into a multi-robot SLAM system for collaborative exploration and pose estimation.






\section*{Appendix}

\begin{algorithm}[t]
\small
\SetKwInOut{Input}{Input}\SetKwInOut{Output}{Output}
\SetKwInOut{Return}{Return}
\caption{MILP-Based Loop Edge Allocation}
\label{alg_milp}
$\mathcal{R}^{\text{inter}} \xleftarrow{} \emptyset$, $\mathcal{S}^{\text{inter}} \xleftarrow{} \emptyset$\\
\For{each $z_{ij}$ in $\mathcal{S}^{*}$ }{
    \eIf{$\mathcal{M}_{\mathcal{R}}(x_{i}) = \mathcal{M}_{\mathcal{R}}(x_{j})$}{
        Insert $z_{ij}$ into $\pvrp_{\mathcal{M}_{\mathcal{R}}(x_{i})}$.\\
    }{
    $\mathcal{R}^{\text{inter}} \xleftarrow{}\mathcal{R}^{\text{inter}} + \mathcal{M}_{\mathcal{R}}(x_{i})$.\\ 
    $\mathcal{R}^{\text{inter}} \xleftarrow{}\mathcal{R}^{\text{inter}} + \mathcal{M}_{\mathcal{R}}(x_{j})$.\\
    $\mathcal{S}^{\text{inter}} \xleftarrow{} \mathcal{S}^{\text{inter}} + z_{ij}.$\\
    }
}
\For{$r$ in $\mathcal{R}^{\text{inter}}$ }{
    $\omega(\mathcal{P}_{r})\xleftarrow{} \omega(\pvrp_{r})$.\\
}

\For{each $z_{ij}$ in $\mathcal{S}^{\text{inter}}$ }{
    Define a Boolean variable $x_{ij}$.\\    $r_{1}\xleftarrow{}\mathcal{M}_{\mathcal{R}}(x_{i})$, $r_{2}\xleftarrow{}\mathcal{M}_{\mathcal{R}}(x_{j})$.\\
    $\omega(\mathcal{P}_{r_{1}})\xleftarrow{} \omega(\mathcal{P}_{r_{1}}) + x_{ij} \cdot \omega(z_{ij}).$\\
    $\omega(\mathcal{P}_{r_{2}})\xleftarrow{} \omega(\mathcal{P}_{r_{2}}) + (1 - x_{ij}) \cdot \omega(z_{ij}).$\\
}
Define a variable $V_{\text{max}}$.\\
Solve following MILP problem:
\begin{equation*}
    \min V_{\text{max}},
\end{equation*}
s.t. $\omega(\mathcal{P}_{r}) \le V_{\text{max}}$, $\forall r\in \mathcal{R}$.\\
Allocate loop edges in $\mathcal{S}^{\text{inter}}$ to robots according to the solution $\{x_{ij}\}_{z_{ij}\in \mathcal{S}^{\text{inter}}}$, and insert to their VRP paths.\\
\KwRet{$\{\mathcal{P}_{r}\}_{r\in\mathcal{R}}$.}\\
\end{algorithm}




\bibliographystyle{ieeetr} 
\bibliography{main} 

\end{document}